\documentclass[conference]{IEEEtran}
\IEEEoverridecommandlockouts
\usepackage{cite}
\usepackage{amsmath,amssymb,amsfonts}
\usepackage{graphicx}
\usepackage{textcomp}
\usepackage{xcolor}
\usepackage[T1]{fontenc}
\usepackage{cite}
\usepackage{amsmath,amssymb,amsfonts}
\usepackage{amsthm}
\usepackage{graphicx}
\usepackage{booktabs}
\usepackage{multirow}
\usepackage{url}
\usepackage{hyperref}
\usepackage{xcolor}
\usepackage{algorithm}
\usepackage{algpseudocode}
\usepackage{balance}

\usepackage{booktabs}
\usepackage{multirow}
\usepackage{tabularx}
\usepackage{booktabs}
\usepackage{multirow}
\usepackage{bm}        
\usepackage{caption}   
\usepackage{algorithm}
\usepackage{algpseudocode}
\newcommand{\eg}{e.g.\ }

\newtheorem{theorem}{Theorem}
\newcommand{\best}[1]{\bm{#1}}

\def\BibTeX{{\rm B\kern-.05em{\sc i\kern-.025em b}\kern-.08em
    T\kern-.1667em\lower.7ex\hbox{E}\kern-.125emX}}
\begin{document}

\title{SpaTeoGL: Spatiotemporal Graph Learning for Interpretable Seizure Onset Zone Analysis from Intracranial EEG\\
\thanks{
This work was conducted during the Master 2 internship of E.~Rostami at Inria Saclay, Palaiseau, France, under the supervision of Prof. T.~Laleg-Kirati.
}
}


\author{
\IEEEauthorblockN{
Elham Rostami\IEEEauthorrefmark{1},
Aref Einizade\IEEEauthorrefmark{2},
and Taous-Meriem Laleg-Kirati\IEEEauthorrefmark{1}
}
\IEEEauthorblockA{\IEEEauthorrefmark{1}
Universit\'e Paris-Saclay, Inria, CIAMS, Gif-sur-Yvette, France\\
elham.rostami@universite-paris-saclay.fr, taous-meriem.laleg-kirati@inria.fr}
\IEEEauthorblockA{\IEEEauthorrefmark{2}
SAMOVAR, Télécom SudParis, Institut Polytechnique de Paris, 91120 Palaiseau, France\\
aref.einizade@telecom-sudparis.eu}
}

\maketitle

\begin{abstract}
Accurate localization of the seizure onset zone (SOZ) from intracranial EEG (iEEG) is essential for epilepsy surgery but is challenged by complex spatiotemporal seizure dynamics. We propose \texttt{SpaTeoGL}, a spatiotemporal graph learning framework for interpretable seizure network analysis. \texttt{SpaTeoGL} jointly learns window-level spatial graphs capturing interactions among iEEG electrodes and a temporal graph linking time windows based on similarity of their spatial structure. The method is formulated within a smooth graph signal processing framework and solved via an alternating block coordinate descent algorithm with convergence guarantees. Experiments on a multicenter iEEG dataset with successful surgical outcomes show that \texttt{SpaTeoGL} is competitive with a baseline based on horizontal visibility graphs and logistic regression, while improving non-SOZ identification and providing interpretable insights into seizure onset and propagation dynamics.
\end{abstract}


\section{Introduction}
Drug-resistant epilepsy (DRE) remains a major neurological burden and often requires surgical intervention when pharmacological treatment fails \cite{deboer2008burden}. A key step in presurgical evaluation is the localization of the \emph{seizure onset zone} (SOZ) from intracranial EEG (iEEG), which is still largely based on expert visual inspection and complementary biomarkers \cite{marin2024signal}. This process is time-consuming, subjective, and difficult to standardize across centers \cite{lucas2024ai_epilepsy, li2023openneuro_ds003029}. Furthermore, seizures are inherently \emph{network events}, with onset and early propagation characterized by evolving multi-electrode interactions rather than isolated signal abnormalities, motivating approaches that jointly capture spatial and temporal structure.

Several quantitative approaches have been proposed to reduce reliance on manual iEEG review. The \emph{Neural Fragility} framework models iEEG dynamics via perturbation sensitivity has demonstrated strong relevance for SOZ localization and surgical outcome's prediction \cite{li2021neuralfragility}. Another line of work focuses on transforming iEEG time series into alternative signal representations from which discriminative features can be extracted. In this context, Semi-Classical Signal Analysis (SCSA) \cite{lalegkirati2013scsa} combined with Visibility Graphs (VG) provides a physics-informed and graph-theoretic representation of seizure-related dynamics \cite{sadoun2024scsa_vg}. Visibility-based representations are particularly attractive as they encode nonlinear temporal structure directly into graph topology, naturally bridging time-series analysis and graph-based modeling \cite{luque2009hvg}.

Recent studies leverage Graph Neural Networks (GNNs) to analyze iEEG-derived graphs, which are typically constructed from functional connectivity measures. Attention-based GNNs can identify seizure-relevant electrodes through learned importance weights \cite{grattarola2022seizuregnn}, while adaptive graph learning approaches aim to jointly infer graph structure and model sequential auto-regressive temporal dynamics (e.g., GCN--LSTM architectures) \cite{guo2025adaptive}. These methods reflect a broader trend in GNN-based modeling \cite{wu2020gnn_survey}, but often require substantial training data, may face scalability limitations, and can lack transparency regarding the specific connectivity patterns driving their predictions and generalization of temporal modeling.

This paper emphasizes the \emph{interpretability} and \emph{analysis} of seizure dynamics, while addressing SOZ classification problem as well. We propose \texttt{SpaTeoGL}, a spatiotemporal graph learning framework that learns a spatial graph for each time window, capturing functional connectivity across brain regions, and a temporal graph across windows, modeling the evolution of seizures over time. \texttt{SpaTeoGL} is inspired by smooth graph signal models \cite{dong2016laplacian} within the Graph Signal Processing (GSP) framework \cite{ortega2018graph, leus2023graph}. Smoothness is enforced in both domains, promoting similar brain activity across connected regions and consistency between windows with similar seizure dynamics. The resulting optimization problem is solved via block coordinate descent (BCD) with theoretical convergence guarantees \cite{tseng2001bcd}. 

{\color{black} Unlike visibility-graph baselines that mainly provide classification features,
\texttt{SpaTeoGL} learns explicit spatial and temporal graph structures that can be inspected for seizure analysis. Compared with supervised GNN-based approaches, it directly optimizes interpretable Laplacians within a smooth graph signal model, reducing reliance on large training sets while facilitating analysis of seizure onset and propagation.}
Building upon the above discussion, the main contributions of this paper are summarized as follows:
(i) a spatiotemporal graph learning framework that jointly estimates window-level spatial graphs and a
temporal graph across windows; (ii) a GSP-based smoothness formulation
solved by BCD with theoretical convergence guarantees; and (iii) an analysis-oriented
evaluation showing competitive SOZ/non-SOZ discrimination and interpretable temporal and spatial seizure patterns compared with \texttt{HVG+LR}.


\textbf{Notation.} Bold uppercase/lowercase letters denote matrices/vectors,
\(\mathrm{vec}(\cdot)\) denotes column-wise vectorization, and
\(\|\cdot\|_F\) is the Frobenius norm. For an undirected weighted graph
\(\mathcal{G}=(\mathcal{V},\mathcal{E},\mathbf{A})\), the combinatorial Laplacian is \(\mathbf{L}=\mathbf{D}-\mathbf{A}\). The feasible
set of valid \(N\)-node Laplacians is
\begin{equation}
\label{eq:Lt_Ls_set}
\begin{split}
&
\mathcal{L}=\{\mathbf{L}\in\mathbb{R}^{N\times N}\mid\\
&\mathbf{L}\succeq0,\ \mathrm{tr}(\mathbf{L})=N,\ L_{ij}=L_{ji}\le0\ (i\ne j),\
\mathbf{L}\mathbf{1}=\mathbf{0}\}.
\end{split}
\end{equation}
The set of graph signals $\{\mathbf{x}_i\in\mathbb{R}^{N}\}_{i=1}^R$ gathered in a matrix $\mathbf{X}\in\mathbb{R}^{N\times R}$, smoothness on \(\mathcal{G}\)
is measured by \(\mathrm{tr}(\mathbf{X}^\top \mathbf{L} \mathbf{X})=\sum_{i\sim j}{A_{ij}\|\mathbf{x}_{(i)} - \mathbf{x}_{(j)}\|_2^2}\) \cite{ortega2018graph}.


\section{SpaTeoGL: Spatiotemporal Graph Learning}
\label{sec:spateogl}


\begin{figure}[!t]
\centering
\includegraphics[width=0.45\textwidth,trim={4cm 0.2cm 7cm 0.5cm},clip=true]{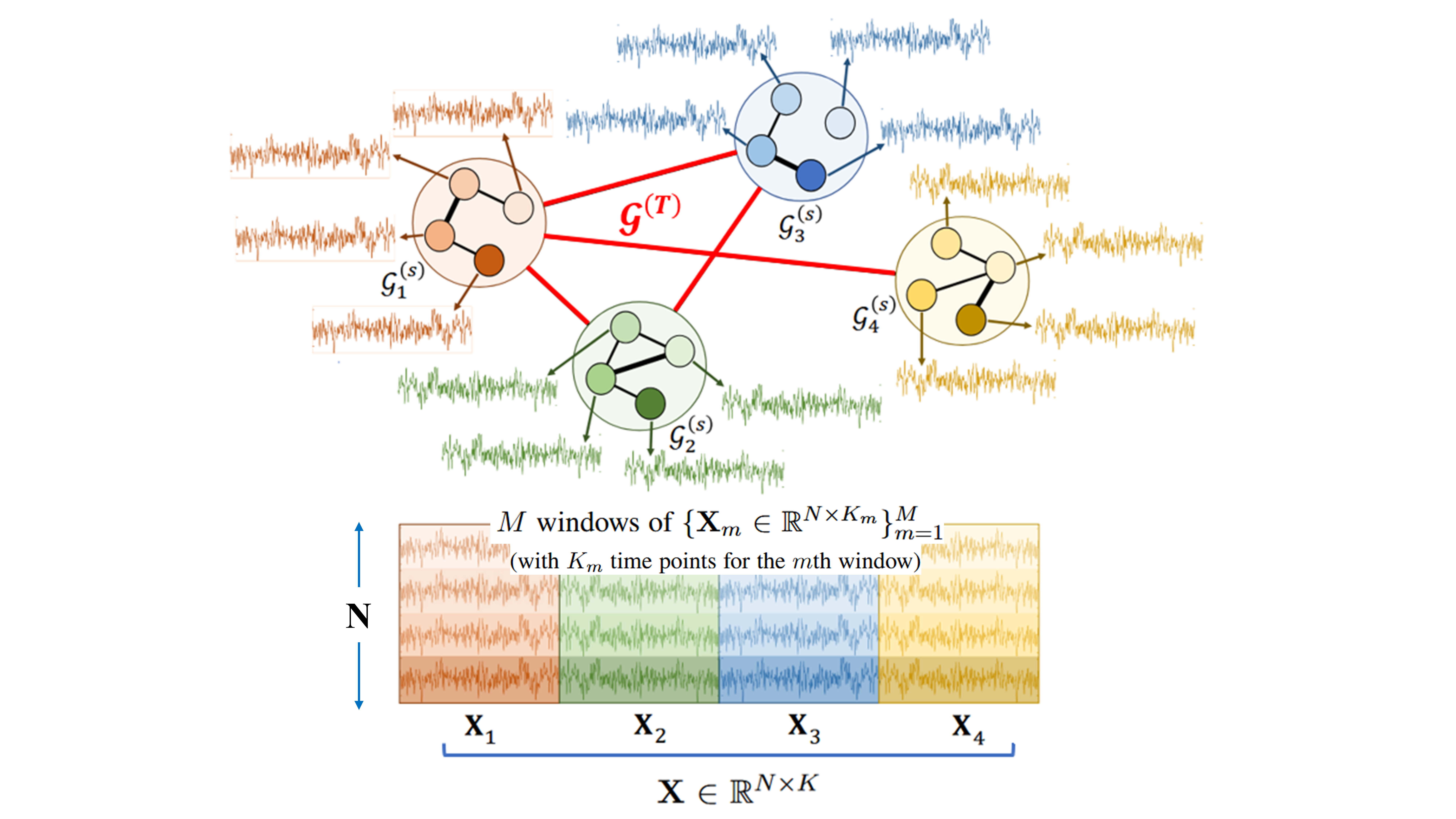}
\caption{Illustration of the \texttt{SpaTeoGL} method in an example setting. Here, the multi-electrode recorded data $\mathbf{X}\in\mathbb{R}^{N\times K}$ with $N=4$ is {\color{black}segmented into $M=4$ windows} of $\{\mathbf{X}_m\in\mathbb{R}^{N\times K_m}\}_{m=1}^M$. Therefore, the {\color{black}\texttt{SpaTeoGL}} learns $M=4$ spatial graphs $\{\mathcal{G}^{(s)}_m\}_{m=1}^M$ representing the connections between electrodes in each window, and one temporal graph $\mathcal{G}^{(T)}$ (with red edges in the figure) showing the connection between $M=4$ windows.} 
\label{fig:SpaTeoGL}
\end{figure}

Fig.~\ref{fig:SpaTeoGL} illustrates the proposed \texttt{SpaTeoGL} setup. 
{\color{black}Let \(\mathbf{X} \in \mathbb{R}^{N \times K}\) denote a multi-electrode iEEG segment, where rows correspond to \(N\) electrodes and columns correspond to \(K\) time samples. The signal is segmented into \(M\) possibly overlapping windows \(\{\mathbf{X}_m\}_{m=1}^{M}\), where \(\mathbf{X}_m \in \mathbb{R}^{N \times K_m}\). When windows overlap, samples contributing to different \(\mathbf{X}_m\) are not disjoint.}
Each column of $\mathbf{X}_m$ is interpreted as a graph signal \cite{ortega2018graph} defined on a \textit{spatial} (electrode) graph associated with window $m$. Thus, for each window, \texttt{SpaTeoGL} learns a spatial graph $\mathcal{G}^{(s)}_m = (\mathcal{V}^{(s)}, \mathcal{E}^{(s)}_m, \mathbf{A}^{(s)}_m)$, where $\mathcal{V}^{(s)}=\{1,\dots,N\}$ indexes the electrodes, and $\mathbf{A}^{(s)}_m\in\mathbb{R}^{N\times N}$ and $\mathbf{L}^{(s)}_m$ denote the corresponding adjacency matrix and Laplacian. Spatial smoothness of the windowed signals is quantified by the quadratic form $\mathrm{tr}(\mathbf{X}_m^\top \mathbf{L}^{(s)}_m \mathbf{X}_m)$ \cite{ortega2018graph}. Across windows, \texttt{SpaTeoGL} also learns a \emph{temporal} graph $\mathcal{G}^{(t)} = (\mathcal{V}^{(t)}, \mathcal{E}^{(t)}, \mathbf{A}^{(t)})$, where $\mathcal{V}^{(t)}=\{1,\dots,M\}$ indexes time windows,
$\mathbf{A}^{(t)}\in\mathbb{R}^{M\times M}$ is the temporal adjacency matrix,
and $\mathbf{L}^{(t)}\in\mathbb{R}^{M\times M}$ is the corresponding Laplacian. Here, the vectorized form of the spatial graph Laplacians $\{\mathbf{L}^{(s)}_m\}_{m=1}^M$ forms a set of smooth graph signals living on the temporal graph $\mathcal{G}_t$. Each window contains graph signals that are assumed smooth on the window’s spatial graph. On the other hand, we concatenate the vectorized form of the spatial Laplacians into a matrix $\tilde{\mathbf{X}}\in \mathbb{R}^{M\times N^2}$ containing the set of smooth graph signals living on the temporal graph $\mathcal{G}_t$ to encode the brain connectivity similarity among connected temporal windows. Hence, the $m$-th row of $\tilde{\mathbf{X}}$ is $\mathrm{vec}(\mathbf{L}^{(s)}_m)^\top$.
Thus, the temporal smoothness term can be interpreted as:
\begin{equation}
\mathrm{tr}\!\Big(\tilde{\mathbf{X}}^\top \mathbf{L}^{(t)} \tilde{\mathbf{X}}\Big)
=
\sum_{i\sim j} w^{(t)}_{ij}\ \big\|\mathrm{vec}(\mathbf{L}^{(s)}_i)-\mathrm{vec}(\mathbf{L}^{(s)}_j)\big\|_2^2,
\label{eq:temporal_interp}
\end{equation}
encouraging windows with similar spatial graphs to be linked in the learned temporal graph. Therefore, we are imposing a notion of joint \textit{spatiotemporal smoothness}. Finally, from smooth graph signal modeling \cite{dong2016laplacian}, the final optimization problem can be cast as a constrained joint spatiotemporal smoothness minimization as follows:
\begin{equation}
\label{eq:joint}
\begin{split}
&\min_{\{\mathbf{L}^{(s)}_m\}_{m=1}^M,\ \mathbf{L}^{(t)}}
\sum_{m=1}^{M}\mathrm{tr}\!\Big(\mathbf{X}_m^\top \mathbf{L}^{(s)}_m \mathbf{X}_m\Big)
+ \beta \sum_{m=1}^{M}\|\mathbf{L}^{(s)}_m\|_F^2\\
&+ \mathrm{tr}\!\Big(\tilde{\mathbf{X}}^\top \mathbf{L}^{(t)} \tilde{\mathbf{X}}\Big)
+ \beta \|\mathbf{L}^{(t)}\|_F^2,\\
&\text{s.t. } \mathbf{L}^{(s)}_m\!\in\!\mathcal{L}^{(s)},\ \mathbf{L}^{(t)}\!\in\!\mathcal{L}^{(t)}.
\end{split}
\end{equation}
Here, building upon the set of valid Laplacians stated in \eqref{eq:Lt_Ls_set}, we constrain the temporal and spatial Laplacians to be in $\mathcal{L}^{(t)}$ and $\mathcal{L}^{(s)}$, respectively. Besides, $\beta>0$ regularizes density/smoothness and ensures strict convexity of each block update (see Theorem~\ref{thm:convergence}). We optimize \eqref{eq:joint} via BCD \cite{tseng2001bcd} with the following details about the subproblems:

\textbf{Spatial update (for each window $m$):} By fixing $\{\mathbf{L}^{(s)}_r\}^M_{r=1,r\ne m}$ and $\mathbf{L}^{(t)}$, the subproblem for optimizing over $\mathbf{L}^{(s)}_m$ takes the following form:
\begin{equation}
\begin{split}
&\min_{\mathbf{L}^{(s)}_m\in\mathcal{L}^{(s)}}
\mathrm{tr}\!\Big(\mathbf{X}_m^\top \mathbf{L}^{(s)}_m \mathbf{X}_m\Big)
+ \beta\|\mathbf{L}^{(s)}_m\|_F^2\\
&+ \sum_{j=1,\ne m}^{M} w^{(t)}_{mj}\ \big\|\mathbf{L}^{(s)}_m-\mathbf{L}^{(s)}_j\big\|_F^2.
\end{split}
\label{eq:spatial_subprob}
\end{equation}
{\color{black}Problem \eqref{eq:spatial_subprob} is strictly convex because the quadratic regularization term \(\beta\|\mathbf{L}_m^{(s)}\|_F^2\), with \(\beta>0\), is strictly convex, while the trace term is linear, the coupling term is convex quadratic, and the feasible set \(\mathcal{L}^{(s)}\) is convex \cite{kalofolias2016learn}. Therefore, the spatial update admits a unique minimizer and can be solved using standard convex optimization solvers, \eg CVXPY \cite{diamond2016cvxpy}.}

\textbf{Temporal update:} To optimize the temporal subproblem over $\mathbf{L}^{(t)}$, the optimization takes the following form:
\begin{equation}
\min_{\mathbf{L}^{(t)}\in\mathcal{L}^{(t)}}
\mathrm{tr}\!\Big(\tilde{\mathbf{X}}^\top \mathbf{L}^{(t)} \tilde{\mathbf{X}}\Big)
+ \beta\|\mathbf{L}^{(t)}\|_F^2.
\label{eq:temporal_subprob}
\end{equation}
where can be cast as the strictly convex problem of learning a graph from the smooth graph signals \cite{dong2016laplacian,kalofolias2016learn}.

The following theorem states the convergence analysis corresponding to the joint multi-convex optimization \eqref{eq:joint}: 

\begin{algorithm}[t]
\caption{SpatioTemporal Graph Leraning (\texttt{SpaTeoGL})}\label{alg:spateogl}
\begin{algorithmic}[1]
\Require iEEG $\mathbf{X}\in\mathbb{R}^{N\times K}$, windows $M$, regularizer $\beta>0$
\State Segment $\mathbf{X}$ into $\{\mathbf{X}_m\}_{m=1}^M$
\State Initialize $\{\mathbf{L}^{(s)}_m\}\in\mathcal{L}^{(s)}$ and $\mathbf{L}^{(t)}\in\mathcal{L}^{(t)}$
\While{not converged}
  \For{$m=1,\dots,M$}
    \State Solve \eqref{eq:spatial_subprob} to update $\mathbf{L}^{(s)}_m$
  \EndFor
  \State Form $\tilde{\mathbf{X}}=[\mathrm{vec}(\mathbf{L}^{(s)}_1)|\dots|\mathrm{vec}(\mathbf{L}^{(s)}_M)]^\top$
  \State Solve \eqref{eq:temporal_subprob} to update $\mathbf{L}^{(t)}$
\EndWhile
\State \Return $\{\mathbf{L}^{(s)}_m\}_{m=1}^M$, $\mathbf{L}^{(t)}$
\end{algorithmic}
\end{algorithm}

\begin{theorem}
\label{thm:convergence}
Assume $\beta>0$ and that each block subproblem \eqref{eq:spatial_subprob} and \eqref{eq:temporal_subprob} is solved exactly over nonempty closed convex sets $\mathcal{L}^{(s)}$ and $\mathcal{L}^{(t)}$.
Then, the objective in \eqref{eq:joint} is non-increasing across \texttt{SpaTeoGL} iterations.
Moreover, every limit point of the iterates is a stationary point of \eqref{eq:joint}.
Each block update admits a unique minimizer due to strict convexity induced by $\beta\|\cdot\|_F^2$.
\end{theorem}

\begin{proof}
First, we fix all blocks except one.
For a spatial block $\mathbf{L}^{(s)}_m$, the objective restricted to $\mathbf{L}^{(s)}_m$ is the sum of:
(i) a linear trace term $\mathrm{tr}(\mathbf{X}_m^\top \mathbf{L}^{(s)}_m \mathbf{X}_m)$,
(ii) a strictly convex quadratic $\beta\|\mathbf{L}^{(s)}_m\|_F^2$ (since $\beta>0$),
and (iii) convex quadratic coupling terms $\sum_j w^{(t)}_{mj}\|\mathbf{L}^{(s)}_m-\mathbf{L}^{(s)}_j\|_F^2$.
Over the closed convex set $\mathcal{L}^{(s)}$, this yields a strictly convex problem with a unique minimizer; thus the spatial update cannot increase \eqref{eq:joint}.
The temporal update is analogous: the trace term is linear in $\mathbf{L}^{(t)}$ and $\beta\|\mathbf{L}^{(t)}\|_F^2$ is strictly convex over closed convex $\mathcal{L}^{(t)}$, hence it also cannot increase the objective and has a unique minimizer. Therefore, each full iteration yields a non-increasing objective sequence bounded below by $0$, so the objective values converge.
Standard BCD convergence results then imply that every limit point is stationary; see \cite{tseng2001bcd}.
\end{proof}

\section{Experimental Protocol and Results}
\label{sec:experiments}

Our primary objective is to evaluate \texttt{SpaTeoGL} as an \emph{analysis tool} for iEEG seizure dynamics by addressing three key questions:
\textbf{Q1:} Is \texttt{SpaTeoGL} competitive with a baseline for SOZ/non-SOZ discrimination?
\textbf{Q2:} Does it reveal temporally coherent regimes (e.g., pre-onset versus propagation) through the learned temporal graph?
\textbf{Q3:} Does it highlight clinically reported SOZ electrodes near seizure onset via the learned spatial graphs?

To address Q1--Q3, we use the Epilepsy-iEEG-Multicenter dataset \cite{li2023openneuro_ds003029}.
We focus on nine patients with successful post-operative outcomes, clearly labeled SOZ electrodes, and usable annotated recordings.
Preprocessing includes bad-electrode removal (based on expert inspection), 60\,Hz notch filtering, 0.5--100\,Hz bandpass filtering, downsampling to 250\,Hz, and windowing around seizure onset (512\,ms windows with 50\% overlap).
Following this pipeline, the analysis is conducted on nine patients comprising 31 recordings in total.
{\color{black}For \texttt{SpaTeoGL}, most recordings yield $M = 13$ spatial graphs, corresponding to three pre-onset, one onset, and nine post-onset windows. For Patient 13, the available temporal signal length results in one fewer window, yielding $M = 12$ spatial
graphs and a temporal graph with 12 nodes.}

\subsection{\textbf{Q1:} Quantitative comparison with baseline}
For baseline comparison, we include \texttt{HVG+LR} (inspired by \cite{sadoun2024scsa_vg}) for SOZ versus non-SOZ electrode classification.
HVG converts each one-dimensional time series into a graph using the horizontal visibility criterion \cite{luque2009hvg}.
To address the resulting high dimensionality, we select three windows per recording (pre-onset, onset, and post-onset), compute the corresponding HVG adjacency matrices, average them, and vectorize the upper-triangular entries.
Principal Component Analysis (PCA) is then applied for dimensionality reduction, followed by logistic regression for classification.

\begin{table}[!t]
\caption{\color{black}Comparison of channel-level SOZ localization performance. Pt. denotes
patient. Class~0 reports non-SOZ specificity, defined as the fraction of
non-SOZ channels correctly identified as non-SOZ. Class~1 reports SOZ
recall/sensitivity, defined as the fraction of clinically labeled SOZ
channels correctly identified as SOZ. Total denotes the overall channel-level
accuracy. Values are reported as mean \(\pm\) standard deviation. The
higher mean between SpaTeoGL and HVG+LR is boldfaced within each class.
The last row reports the across-patient mean \(\pm\) standard deviation.
The \(p\)-values were obtained using an independent \(t\)-test after
verifying normality of the compared score distributions. 
SpaTeoGL significantly outperforms HVG+LR for Class~0/non-SOZ
identification \((p=0.0048)\), whereas no statistically significant
difference is observed for Class~1/SOZ identification \((p=0.6404)\).}
\label{tab:summary_patients_single}
\centering
\scriptsize
\setlength{\tabcolsep}{2pt}
\renewcommand{\arraystretch}{1.10}

\resizebox{\columnwidth}{!}{%
\begin{tabular}{c cc cc cc}
\toprule
\multirow{2}{*}{\textbf{Pt.}} &
\multicolumn{2}{c}{\textbf{Class 0 (non-SOZ)}} &
\multicolumn{2}{c}{\textbf{Class 1 (SOZ)}} &
\multicolumn{2}{c}{\textbf{Total}} \\
\cmidrule(lr){2-3}\cmidrule(lr){4-5}\cmidrule(lr){6-7}
& \texttt{SpaTeoGL} & \texttt{HVG+LR}
& \texttt{SpaTeoGL} & \texttt{HVG+LR}
& \texttt{SpaTeoGL} & \texttt{HVG+LR} \\
\midrule

01 & $\best{0.82 \pm 0.06}$ & $0.63 \pm 0.10$ & $\best{0.77 \pm 0.11}$ & $0.65 \pm 0.11$ & $\best{0.81 \pm 0.06}$ & $0.63 \pm 0.07$ \\

02 & $0.75 \pm 0.05$ & $\best{0.80 \pm 0.06}$ & $\best{0.42 \pm 0.15}$ & $0.33 \pm 0.20$ & $0.68 \pm 0.05$ & $\best{0.70 \pm 0.01}$ \\

06 & $\best{0.84 \pm 0.11}$ & $0.49 \pm 0.38$ & $0.55 \pm 0.14$ & $\best{0.63 \pm 0.45}$ & $\best{0.76 \pm 0.12}$ & $0.52 \pm 0.16$ \\

08 & $\best{0.83 \pm 0.08}$ & $0.66 \pm 0.22$ & $\best{0.64 \pm 0.09}$ & $0.52 \pm 0.30$ & $\best{0.77 \pm 0.03}$ & $0.62 \pm 0.07$ \\

10 & $0.56 \pm 0.13$ & $\best{0.71 \pm 0.15}$ & $\best{0.63 \pm 0.12}$ & $0.34 \pm 0.27$ & $0.57 \pm 0.11$ & $\best{0.63 \pm 0.09}$ \\

11 & $\best{0.74 \pm 0.15}$ & $0.58 \pm 0.21$ & $0.27 \pm 0.18$ & $\best{0.59 \pm 0.16}$ & $0.55 \pm 0.13$ & $\best{0.58 \pm 0.11}$ \\

13 & $\best{0.96 \pm 0.04}$ & $0.81 \pm 0.17$ & $\best{0.50 \pm 0.12}$ & $0.30 \pm 0.22$ & $\best{0.89 \pm 0.03}$ & $0.74 \pm 0.12$ \\

15 & $\best{0.80 \pm 0.15}$ & $0.71 \pm 0.05$ & $0.21 \pm 0.05$ & $\best{0.51 \pm 0.14}$ & $0.58 \pm 0.11$ & $\best{0.62 \pm 0.06}$ \\

16 & $\best{0.73 \pm 0.06}$ & $0.53 \pm 0.03$ & $0.43 \pm 0.23$ & $\best{0.66 \pm 0.37}$ & $\best{0.66 \pm 0.10}$ & $0.54 \pm 0.11$ \\

\midrule
\textbf{$m\pm\sigma$} &
$\best{0.78 \pm 0.10}$ &
$0.66 \pm 0.11$ &
$0.49 \pm 0.17$ &
$\best{0.50 \pm 0.14}$ &
$\best{0.70 \pm 0.11}$ &
$0.62 \pm 0.07$ \\
\bottomrule
\end{tabular}%
}
\end{table}

\begin{figure}[!t]
\centering
\includegraphics[width=0.47\textwidth]{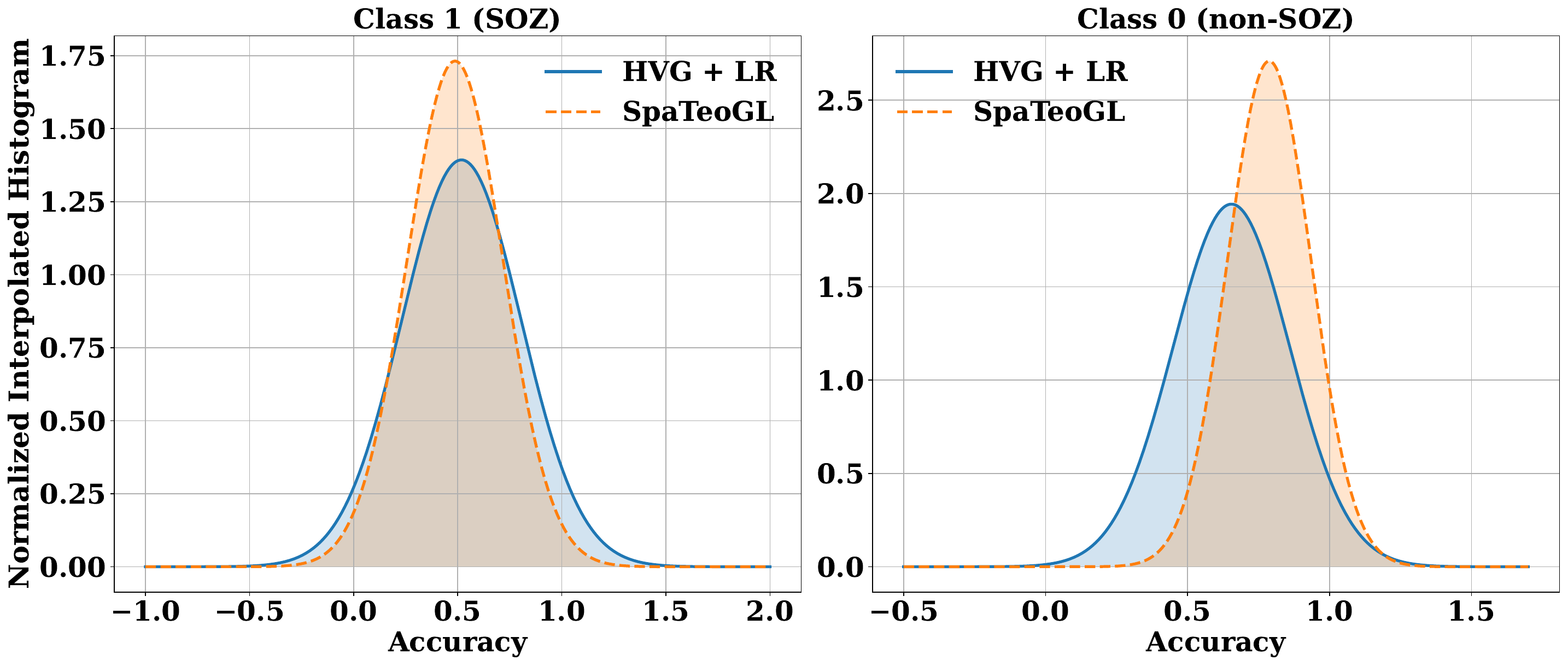}
\caption{Comparison of Normalized Interpolated Histogram for \texttt{HVG+LR} and \texttt{SpaTeoGL} in classifying SOZ (left) and 
non-SOZ (right) electrodes. While no statistically significant difference was found for SOZ electrodes (p = 0.6404), a significant difference was observed for non-SOZ ones (p = 0.0048), with \texttt{SpaTeoGL} showing superior performance.}
\label{fig:Hist}
\end{figure}

\begin{figure}[!t]
    \centering
\includegraphics[width=0.45\textwidth]{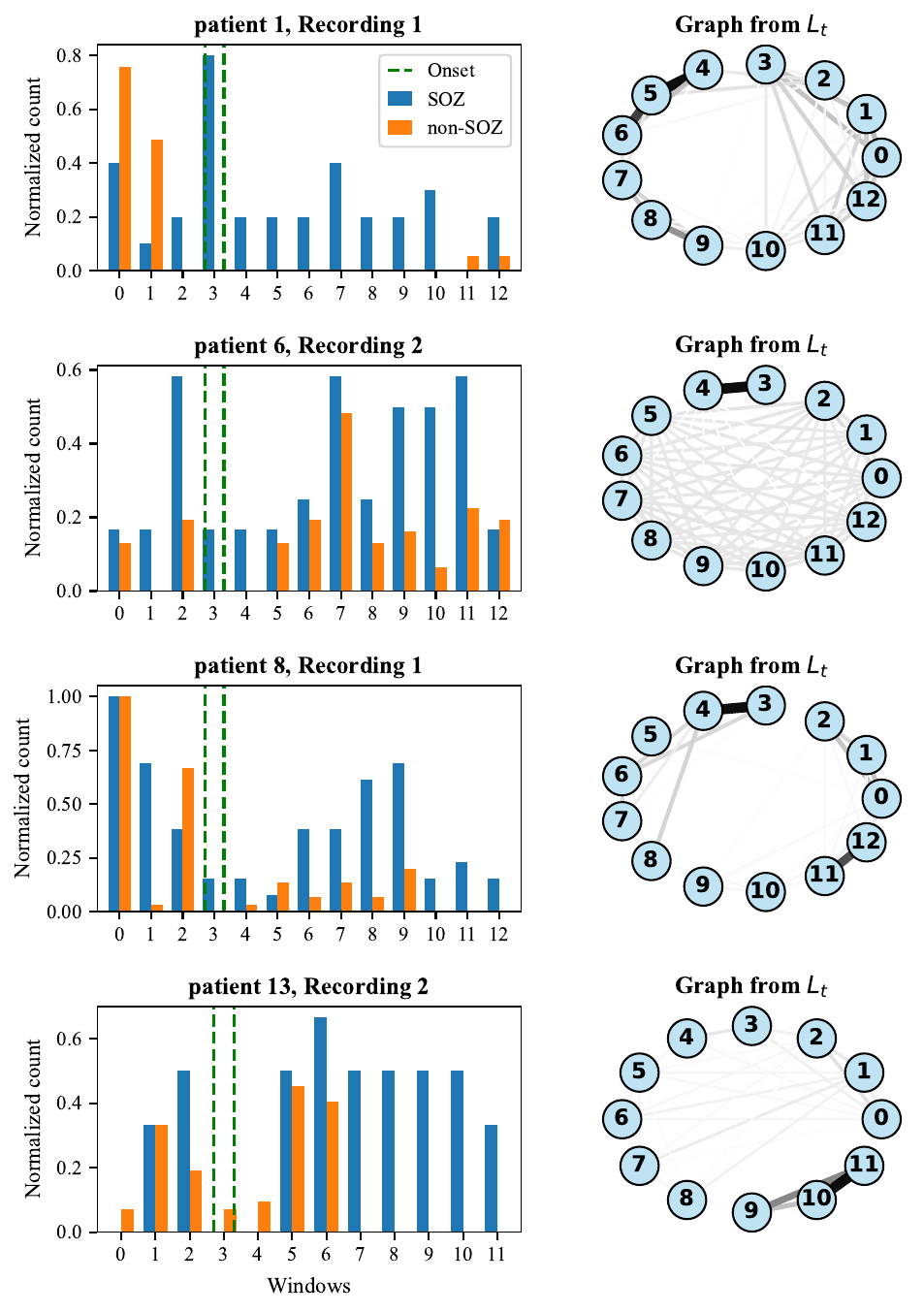}
\caption{\color{black}Temporal evolution of learned spatial graphs and corresponding learned temporal graphs for representative recordings. Left: for each time window, bars show the normalized proportion of selected high-connectivity electrodes belonging to the clinically labeled SOZ and non-SOZ sets. The vertical dotted line indicates the annotated seizure onset. Right: learned temporal graph \(\mathcal{G}^{(T)}\), where each node corresponds to a time window and edge weights reflect similarity between the learned spatial Laplacians of the corresponding windows. Stronger edges indicate more similar spatial connectivity patterns. The first three recordings contain \(M=13\) windows, whereas, for Patient 13, Recording 2 contains \(M=12\) ones.}
\label{fig:temp_graphs}
\end{figure}

\begin{figure}[!t]
    \centering
    \includegraphics[width=0.3\textwidth]{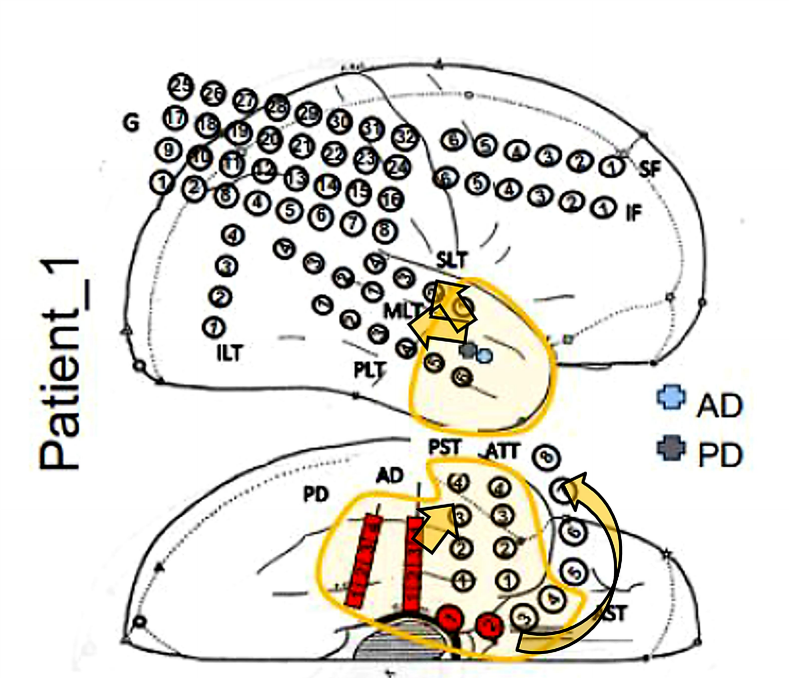}
    \caption*{(a)}
    
    \vspace{0.1cm}    \includegraphics[width=0.48\textwidth,clip=true,trim={0 1.5cm 0 0}]{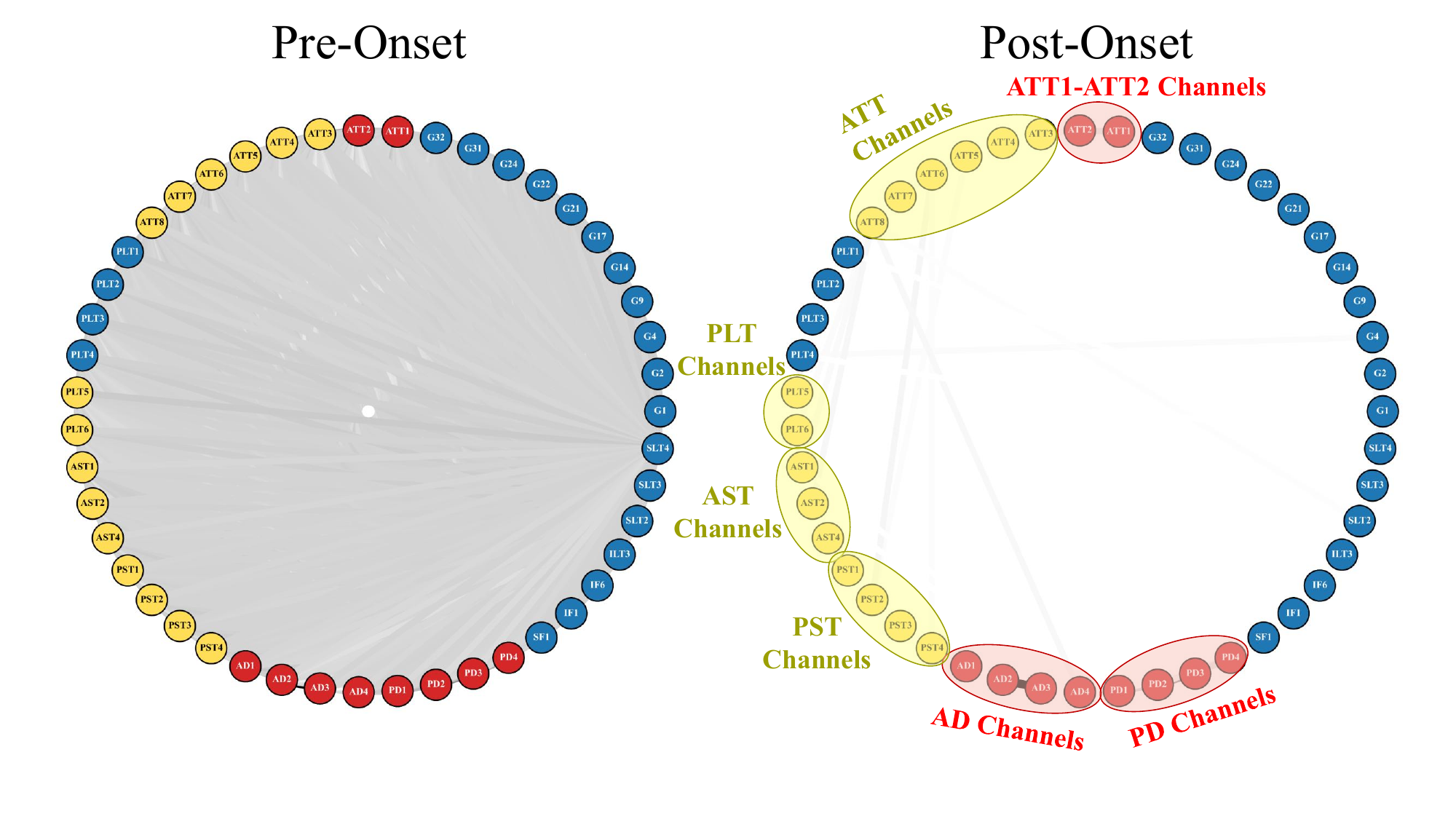}

    \caption*{(b)}
    
    \caption{(a) Electrode placement map on the brain and the potential propagation pathway – patient (pt.) 01 (from \cite{li2023openneuro_ds003029}). (b) Illustration of the difference between mean spatial connections corresponding to Pre-Onset and Post-Onset intervals - pt. 01.}
    \label{fig:spat_graph}
\end{figure}

{\color{black}For quantitative comparison, we report class-wise channel identification scores in Table~\ref{tab:summary_patients_single}. Class 1 corresponds to SOZ channels, and its score is defined as the fraction of clinically labeled SOZ channels that are correctly identified by the method, i.e., SOZ recall/sensitivity. Class 0 corresponds to non-SOZ channels, and its score is defined as the fraction of non-SOZ channels that are correctly identified as non-SOZ, i.e., non-SOZ specificity. The Total column reports the overall channel-level accuracy}. This table indicates that \texttt{SpaTeoGL} is more conservative in labeling electrodes as seizure-related, thereby reducing false positives and improving non-SOZ (class~0) identification.
Across all recordings, \texttt{SpaTeoGL} outperforms \texttt{HVG+LR} on non-SOZ classification, while achieving comparable performance for SOZ (class~1) detection.
The statistical comparison shown in Fig.~\ref{fig:Hist}, based on normalized interpolated histograms, confirms no significant difference for SOZ classification ($p=0.6404$) but a significant advantage for non-SOZ classification ($p=0.0048$).
This results in an overall accuracy of 70\% for \texttt{SpaTeoGL}, compared to 62\% for \texttt{HVG+LR} (Table~\ref{tab:summary_patients_single}).

\subsection{\textbf{Q2:} Temporal graph learning for revealing temporally coherent regimes (onset vs. propagation)}

An important feature of the \texttt{SpaTeoGL} framework is its ability to characterize seizure onset and subsequent propagation through the learned temporal graph.
The right column of Fig.~\ref{fig:temp_graphs} shows the temporal graphs learned from 13 windows (three pre-onset, one onset, and nine post-onset) for representative recordings from patients 1, 6, 8, and 13. {\color{black}The first three recordings contain $M = 13$
windows, whereas, for Patient 13, Recording 2 contains $M = 12$ ones due to its shorter available temporal signal length.}
In these cases, stronger connections between onset and post-onset windows are clearly observed (e.g., among windows 3--5 for patients 1, 6, and 8, {\color{black} windows 9–12 for Patient 8, and windows 8–11 for Patient 13}).
These patterns indicate that \texttt{SpaTeoGL} effectively captures the transition from seizure onset to propagation.

Around seizure onset, seizure-related electrodes are expected to dominate or exhibit stronger connectivity than non-seizure-related electrodes.
To quantify this behavior, the bar plots in the left column of Fig.~\ref{fig:temp_graphs} report, for each window, two ratios:
(i) the proportion of clinically reported seizure-related (SOZ) electrodes present in the spatial graph (blue), and (ii) the proportion of non-seizure-related (non-SOZ) electrodes (orange). Electrodes are identified based on the highest node degree and the largest sum of incident edge weights. Overall, SOZ electrodes tend to dominate at seizure onset and persist across subsequent windows, illustrating the temporal analysis capability of the \texttt{SpaTeoGL} framework.

\subsection{\textbf{Q3:} Spatial graph learning to localize clinically reported SOZ electrodes}
To further illustrate the spatial SOZ localization capability of \texttt{SpaTeoGL}, we focus on Patient~1, for whom detailed seizure propagation information is available in the dataset (Fig.~\ref{fig:spat_graph}(a)). The red-highlighted electrodes (PD, AD, and ATT1--2) correspond to the clinically identified seizure onset regions, followed by propagation to adjacent cortical areas (PST1--4, ATT3--8, AST1--4, PLT5--6, and SLT1). {\color{black} Note that electrode names follow the clinical electrode naming system provided in the OpenNeuro dataset \cite{li2023openneuro_ds003029}.} To assess whether \texttt{SpaTeoGL} captures this pattern, we average the learned spatial graphs over pre-onset and post-onset windows (Fig.~\ref{fig:spat_graph}(b)). 
{\color{black}In the post-onset graph, the strongest average spatial connections are concentrated around contacts clinically labeled as SOZ, shown in red. Additional moderate connections appear around neighboring contacts labeled as propagation-related regions, shown in yellow. In contrast, the pre-onset graph displays weaker and less spatially organized connectivity. This qualitative pattern suggests that SpaTeoGL captures the transition from background activity to a more localized seizure-related network after onset, in agreement with the clinical annotations shown in Fig.~\ref{fig:spat_graph}(a).}

\section{Conclusion}
We introduced \texttt{SpaTeoGL}, a spatiotemporal graph learning framework for interpretable SOZ analysis from iEEG. The method jointly learns window-level spatial graphs and a temporal graph across windows, and is solved via an alternating optimization scheme with convergence guarantees. Experimental results show that \texttt{SpaTeoGL} improves non-SOZ identification compared to a strong baseline (\texttt{HVG+LR}), while providing meaningful insights into seizure onset regimes and propagation dynamics through its learned spatiotemporal graph structure.

\section*{Acknowledgment}
The contributions of Aref Einizade to this work were supported by the French National Research Agency (ANR) 2030 plan under grant agreement ANR-23-CMAS-0033.

\balance
\bibliographystyle{IEEEtran}
\bibliography{refs}

\end{document}